\documentclass[sigconf]{aamas}

\usepackage{balance} % for balancing columns on the final page
\usepackage{amsthm}
\usepackage{algorithm}          % 浮动体（带 caption/label）
\usepackage{algpseudocode}      % algorithmicx 的伪代码环境（支持 \State, \For, ...）

\algrenewcommand{\algorithmicrequire}{Input:}
\algrenewcommand{\algorithmicensure}{Output:}
\algtext*{EndWhile}\algtext*{EndFor}\algtext*{EndIf}

\newtheorem{theorem}{Theorem}

\setcopyright{ifaamas}
\acmConference[AAMAS '26]{Proc.\@ of the 25th International Conference
on Autonomous Agents and Multiagent Systems (AAMAS 2026)}{May 25 -- 29, 2026}
{Paphos, Cyprus}{C.~Amato, L.~Dennis, V.~Mascardi, J.~Thangarajah (eds.)}
\copyrightyear{2026}
\acmYear{2026}
\acmDOI{}
\acmPrice{}
\acmISBN{}

\title[AAMAS-2026 Formatting Instructions]{Mean Field Game-Based Interactive Trajectory Planning Using Physics-Inspired Unified Potential Fields}

\author{Zhen Tian}
\affiliation{
  \institution{School of Computing Science, University of Glasgow}
  \city{Glasgow}
  \country{United Kingdom}}
\email{Tian.Zhen@glasgow.ac.uk}

\author{Fujiang Yuan}
\affiliation{
  \institution{College of Mechanical Engineering, Chongqing University of Technology, Chongqing, 400054, China.}
  \city{Chongqing}
      \country{China}}
\email{yuanfujiang@ctbu.edu.cn}

\author{Chunhong Yuan}
\affiliation{
  \institution{Laboratory of Intelligent Home Appliances, College of Science and Technology, Ningbo University, Ningbo 315300, Zhejiang, China.}
  \city{Zhejiang}
      \country{China}}
\email{ChYuan@stud.kpfu.ru}

\author{Yanhong Peng}
\affiliation{
  \institution{College of Mechanical Engineering, Chongqing University of Technology, Chongqing, 400054, China.}
  \city{Chongqing}
      \country{China}}
\email{yhpeng@nagoya-u.jp}
\begin{abstract}
Interactive trajectory planning in autonomous driving must balance safety, efficiency, and scalability under heterogeneous driving behaviors. Existing methods often face high computational cost or rely on external safety critics. To address this, we propose an Interaction-Enriched Unified Potential Field (IUPF) framework that fuses style-dependent benefit and risk fields through a physics-inspired variational model, grounded in mean field game theory. The approach captures conservative, aggressive, and cooperative behaviors without additional safety modules, and employs stochastic differential equations to guarantee Nash equilibrium with exponential convergence. Simulations on lane changing and overtaking scenarios show that IUPF ensures safe distances, generates smooth and efficient trajectories, and outperforms traditional optimization and game-theoretic baselines in both adaptability and computational efficiency.
\end{abstract}


\newcommand{\BibTeX}{\rm B\kern-.05em{\sc i\kern-.025em b}\kern-.08em\TeX}

\begin{document}

%%% The following commands remove the headers in your paper. For final 
%%% papers, these will be inserted during the pagination process.

\pagestyle{fancy}
\fancyhead{}

%%% The next command prints the information defined in the preamble.

\maketitle 

%%%%%%%%%%%%%%%%%%%%%%%%%%%%%%%%%%%%%%%%%%%%%%%%%%%%%%%%%%%%%%%%%%%%%%%%

\section{Introduction}
\begin{figure}[t]
    \centering
    \includegraphics[width=0.95\linewidth]{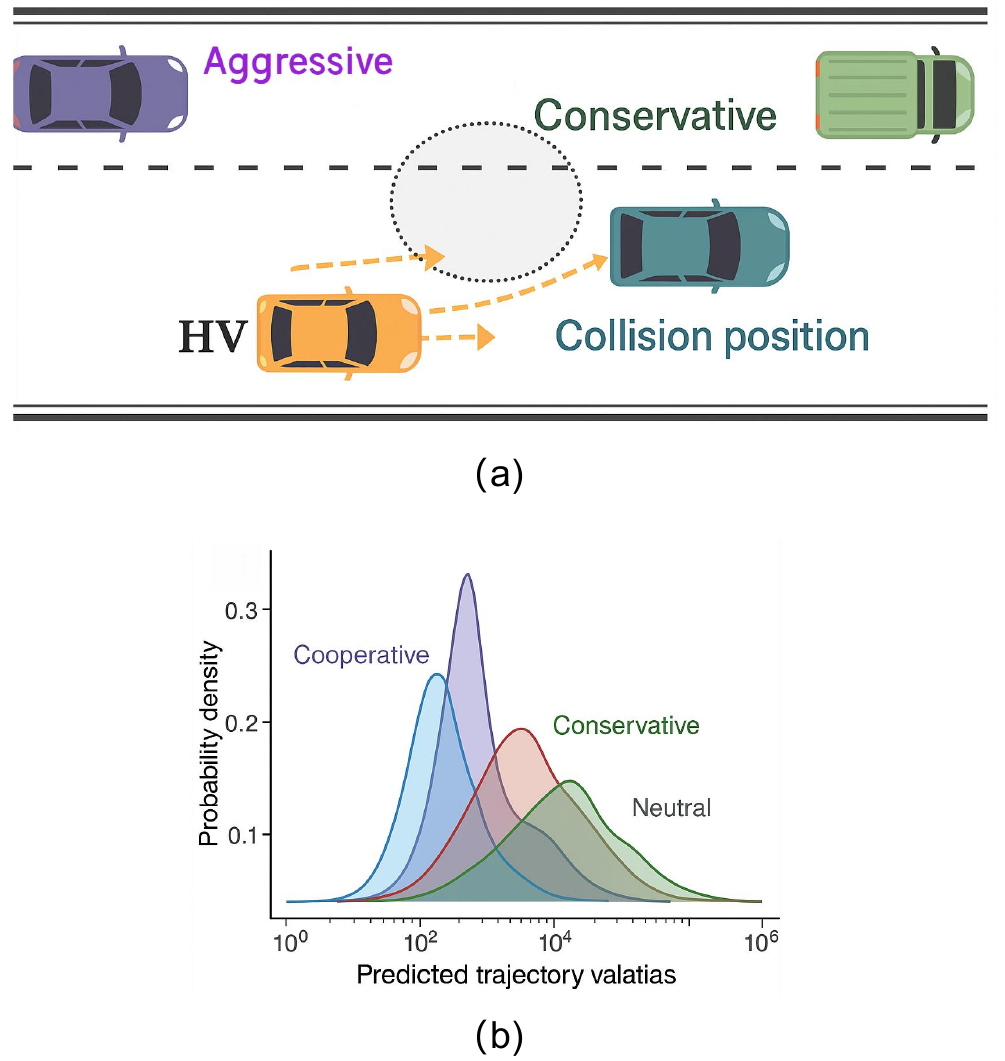}
    \caption{(a) Interaction between the HV and surrounding vehicles with different driving styles, where both aggressive and conservative behaviors influence the potential collision region. 
    (b) Distribution of predicted trajectory tendencies across different driver styles, showing the heterogeneity in risk and interaction patterns.}
    \label{fig:interaction_styles}
\end{figure}
Autonomous driving has made significant progress in recent years, with perception and localization technologies becoming increasingly mature. However, ensuring safe and efficient operation in interactive traffic environments still relies heavily on the quality of the decision-making and trajectory planning modules~\cite{yuan2025bio, 9750986, lin2024dpl, lin2024enhanced, lin2025slam2}. During driving, planning must generate smooth, safe, and dynamically feasible trajectories by balancing multiple objectives such as efficiency, safety, and comfort~\cite{liu2021dynamic, li2025efficient}. This balancing act becomes particularly challenging in dense interactive scenarios, where the host vehicle (HV) must negotiate with surrounding agents whose intentions and driving styles vary significantly. As illustrated in Fig.~\ref{fig:interaction_styles}, our framework explicitly models heterogeneous driving styles. Fig.~\ref{fig:interaction_styles}(a) shows how the host vehicle (HV) interacts with aggressive and conservative vehicles in a lane-change scenario, while Fig.~\ref{fig:interaction_styles}(b) demonstrates the probabilistic distribution of predicted trajectories under different styles. 
This highlights the necessity of incorporating style-awareness into mean-field decision making.

\subsection{Game-theoretic and Mean-Field Approaches in Interactive Driving}
Beyond classical rule-based and optimization-based methods~\cite{9729796, 9964675, li2025adaptive, lin2025multi, lin2025safety}, a growing body of research has adopted game-theoretic reasoning to address interactive decision making. By formulating driving as a multi-agent game, each vehicle is treated as a player optimizing its own utility, allowing the capture of strategic interactions in lane changing, merging, and intersection scenarios~\cite{wang2022hybrid, 9560056}. Nash equilibrium and Stackelberg games have been widely explored to model interactions; however, these approaches suffer from computational burdens and scalability issues as the number of agents increases~\cite{zhang2023path, 10004211}. 
\begin{figure*}[t]
    \centering
    \includegraphics[width=0.7\linewidth]{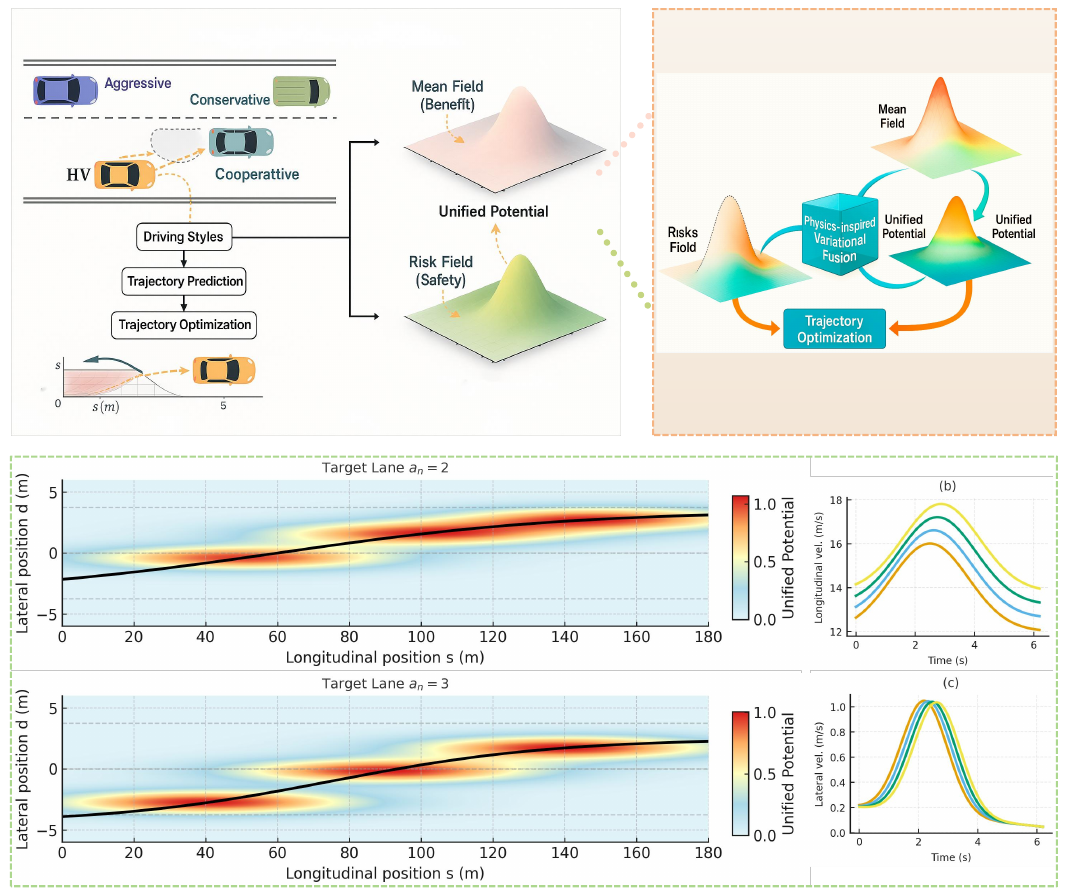}
    \caption{Overview of the proposed IUPF framework.}
    \label{fig:framework_overview}
\end{figure*}
More recently, mean-field game (MFG) theory has been introduced into autonomous driving as a scalable solution to the multi-agent interaction problem~\cite{wang2022hybrid,9560056,zhang2023path}. By approximating the aggregate effect of a large population with a distribution, MFGs provide a tractable framework that maintains game-theoretic consistency while reducing complexity. This approach allows interactive planning without explicitly modeling each surrounding vehicle. However, most existing MFG-based methods either assume homogeneous driving behaviors, or require additional complex safety supervision (e.g., TTC thresholds, control barrier certificates, or multi-layer critics) to ensure collision avoidance~\cite{zheng2025mean,zheng2025enhanced}. These additional modules not only increase design complexity but may also introduce conservativeness or inconsistencies in dynamic environments.

\subsection{Limitations of Existing Approaches}
Despite their advances, current planning methods still face several limitations. Classical approaches such as graph search, sampling-based methods, artificial potential fields, and optimization frameworks provide useful baselines but often encounter issues of high computation cost, suboptimal behavior in dense traffic, or sensitivity to handcrafted objectives~\cite{9942280, 9756640, 9535367, 9830995, 9967433, 9415170, 10048483, 10285563}. Game-theoretic models capture interaction but do not scale well and require explicit modeling of all participants, while mean-field methods improve scalability but tend to oversimplify heterogeneous driver behaviors. Moreover, many of these methods rely on external safety critics such as TTC- or DRAC-based checks~\cite{zhang2023path, yang2023decoupled}, resulting in heavy supervision overhead. This limits their applicability to highly dynamic and interactive real-world traffic. Furthermore, most validations focus on simplified settings with few vehicles or constrained scenarios, leaving open the question of robustness and generalizability in complex multi-vehicle environments~\cite{zhang2023path, 10004211}.

In light of these challenges, this paper introduces an IUPF framework that integrates the advantages of mean-field modeling with physically inspired risk shaping. Instead of separating interaction modeling from safety supervision, the proposed framework combines a mean-field benefit field with a risk field into a single unified potential, enabling interaction-rich yet safety-aware decision making without the need for multiple external critics. Driver heterogeneity is incorporated through style-dependent field parameters, ensuring the planner can capture both conservative and aggressive behaviors within the same formulation. This integration reduces the reliance on handcrafted supervision while maintaining safety margins, resulting in smooth, efficient, and interpretable maneuvers such as lane changes and overtaking. Comparative evaluations demonstrate that the proposed IUPF achieves more robust and efficient decision making compared to existing optimization- and game-based methods, confirming its promise for scalable deployment in interactive autonomous driving. The main contributions of this work are summarized as follows:
\begin{itemize}
    \item We propose an IUPF that integrates mean-field benefits and risk shaping into a single potential, enabling interaction-rich yet safety-aware decision making.
    \item Driver heterogeneity is explicitly captured through style-dependent field parameters, allowing the framework to handle conservative and aggressive behaviors within a unified formulation.
    \item The framework achieves smooth, efficient, and interpretable maneuvers without relying on multiple external safety critics, and demonstrates superior robustness in complex lane-change scenarios.
\end{itemize}
%%%%%%%%%%%%%%%%%%%%%%%%%%%%%%%%%%%%%%%%%%%%%%%%%%%%%%%%%%%%%%%%%%%%%%%%

\section{The Proposed Framework}

The proposed framework, as shown in Fig.~\ref{fig:framework_overview}, provides a unified representation for decision making in interactive driving scenarios. The HV is required to interact with surrounding vehicles that may exhibit different driving styles such as aggressive, conservative, or cooperative. To address this challenge, we design a two-stage field construction process. First, a Mean Field is established to capture the potential benefits of the surrounding vehicles' motion trends, while a Risk Field models the safety exposure that reflects the likelihood of collisions. These two fields represent complementary aspects of the environment: the Mean Field promotes efficiency and opportunity, whereas the Risk Field emphasizes safety and caution.

The key innovation lies in introducing a physics-inspired variational fusion model that combines the Mean Field and Risk Field into a single Unified Potential Field. This fusion mechanism inherently balances the trade-off between safety and efficiency, avoiding the need for multiple external critics or additional safety heuristics such as explicit time-to-collision checks. As a result, the Unified Potential Field simultaneously encodes both risk and benefit, enabling the planner to adapt to different interaction patterns without additional supervision.

Based on the Unified Potential Field, the trajectory planning process exploits both the field distribution and its gradient to generate feasible maneuvers. In the illustrated lane-change case study, the unified potential highlights safe and beneficial regions along candidate lanes, guiding the host vehicle to select the most suitable path. Furthermore, the velocity profiles confirm that the proposed framework not only ensures safety and efficiency but also maintains smooth control performance across longitudinal and lateral directions. This demonstrates the robustness and adaptability of the approach when handling heterogeneous driving styles in complex scenarios.

%%%%%%%%%%%%%%%%%%%%%%%%%%%%%%%%%%%%%%%%%%%%%%%%%%%%%%%%%%%%%%%%%%%%%%%%
\section{Methodology: Interaction-Enriched Unified Potential Field Framework}
\label{sec:methodology}

This section presents a mathematically rigorous framework for interactive trajectory planning through mean field games and variational potential field fusion.

\subsection{Vehicle Dynamics in Frenet Coordinate System}

We adopt a discrete-time kinematic bicycle model in the Frenet coordinate system. The state vector at time step $k$ is defined as:

\begin{equation}
\mathbf{x}_k = \begin{bmatrix} 
s_k \\ d_k \\ \dot{s}_k \\ \dot{d}_k \\ \ddot{s}_k \\ \ddot{d}_k 
\end{bmatrix} \in \mathbb{R}^6
\label{eq:state_vector}
\end{equation}

where $s_k \in \mathbb{R}$ represents the longitudinal position along the reference path at time $k$, $d_k \in \mathbb{R}$ denotes the lateral displacement from the reference path, $\dot{s}_k, \dot{d}_k \in \mathbb{R}$ are the corresponding longitudinal and lateral velocities, and $\ddot{s}_k, \ddot{d}_k \in \mathbb{R}$ represent the longitudinal and lateral accelerations respectively.

The control input vector is:
\begin{equation}
\mathbf{u}_k = \begin{bmatrix} 
a_{s,k} \\ \omega_{d,k} 
\end{bmatrix} \in \mathcal{U} \subset \mathbb{R}^2
\label{eq:control_vector}
\end{equation}

where $a_{s,k} \in [-a_{\max}, a_{\max}]$ represents the longitudinal jerk (rate of acceleration change) and $\omega_{d,k} \in [-\omega_{\max}, \omega_{\max}]$ denotes the lateral acceleration rate, both subject to physical actuator constraints.

The linearized discrete-time dynamics follow:
\begin{equation}
\mathbf{x}_{k+1} = \mathbf{A} \mathbf{x}_k + \mathbf{B} \mathbf{u}_k + \mathbf{w}_k
\label{eq:dynamics}
\end{equation}

where the system matrices are:
\begin{equation}
\mathbf{A} = \begin{bmatrix}
1 & 0 & \Delta t & 0 & \frac{\Delta t^2}{2} & 0 \\
0 & 1 & 0 & \Delta t & 0 & \frac{\Delta t^2}{2} \\
0 & 0 & 1 & 0 & \Delta t & 0 \\
0 & 0 & 0 & 1 & 0 & \Delta t \\
0 & 0 & 0 & 0 & 1 & 0 \\
0 & 0 & 0 & 0 & 0 & 1
\end{bmatrix}, \quad
\mathbf{B} = \begin{bmatrix}
\frac{\Delta t^3}{6} & 0 \\
0 & \frac{\Delta t^3}{6} \\
\frac{\Delta t^2}{2} & 0 \\
0 & \frac{\Delta t^2}{2} \\
\Delta t & 0 \\
0 & \Delta t
\end{bmatrix}
\label{eq:system_matrices}
\end{equation}

and $\mathbf{w}_k \sim \mathcal{N}(\mathbf{0}, \mathbf{Q})$ represents process noise with covariance matrix:
\begin{equation}
\mathbf{Q} = \sigma_w^2 \begin{bmatrix}
\frac{\Delta t^4}{4} & 0 & \frac{\Delta t^3}{2} & 0 & \frac{\Delta t^2}{2} & 0 \\
0 & \frac{\Delta t^4}{4} & 0 & \frac{\Delta t^3}{2} & 0 & \frac{\Delta t^2}{2} \\
\frac{\Delta t^3}{2} & 0 & \Delta t^2 & 0 & \Delta t & 0 \\
0 & \frac{\Delta t^3}{2} & 0 & \Delta t^2 & 0 & \Delta t \\
\frac{\Delta t^2}{2} & 0 & \Delta t & 0 & 1 & 0 \\
0 & \frac{\Delta t^2}{2} & 0 & \Delta t & 0 & 1
\end{bmatrix}
\label{eq:noise_covariance}
\end{equation}

where $\sigma_w > 0$ denotes the process noise intensity and $\Delta t > 0$ is the sampling time interval.

\subsection{Mean Field Game Formulation with Heterogeneous Driving Styles}

Consider a population of $N$ vehicles, where each vehicle $i \in \{1, 2, \ldots, N\}$ is characterized by its state $\mathbf{x}_t^i \in \mathbb{R}^6$ and driving style $\theta^i \in \Theta$. The driving style space is defined as:

\begin{equation}
\Theta = \{\theta_c, \theta_a, \theta_{co}\} \times [0,1]^3
\label{eq:style_space}
\end{equation}

where $\theta_c$ represents conservative behavior characterized by risk-averse parameters, $\theta_a$ denotes aggressive behavior with efficiency-focused parameters, and $\theta_{co}$ represents cooperative behavior with balanced parameters. The continuous components $[0,1]^3$ parameterize the aggressiveness level, reaction time, and social awareness respectively.

Each vehicle evolves according to the controlled stochastic differential equation:
\begin{equation}
d\mathbf{X}_t^i = \mathbf{f}(\mathbf{X}_t^i, \mathbf{u}_t^i, m_t, \theta^i) dt + \boldsymbol{\Sigma}(\mathbf{X}_t^i, \theta^i) d\mathbf{W}_t^i
\label{eq:sde}
\end{equation}

where $\mathbf{f}: \mathbb{R}^6 \times \mathbb{R}^2 \times \mathcal{P}(\mathbb{R}^6) \times \Theta \to \mathbb{R}^6$ represents the drift coefficient, $\mathbf{u}_t^i \in \mathcal{U}(\theta^i) \subset \mathbb{R}^2$ is the control input constrained by driving style, $m_t \in \mathcal{P}(\mathbb{R}^6)$ denotes the empirical measure of vehicle distribution, $\boldsymbol{\Sigma}: \mathbb{R}^6 \times \Theta \to \mathbb{R}^{6 \times 6}$ is the diffusion coefficient matrix, and $\mathbf{W}_t^i$ are independent 6-dimensional Brownian motions.

The drift coefficient incorporates vehicle dynamics and mean field interaction:
\begin{equation}
\mathbf{f}(\mathbf{x}, \mathbf{u}, m, \theta) = \mathbf{A} \mathbf{x} + \mathbf{B} \mathbf{u} + \boldsymbol{\phi}(\mathbf{x}, m, \theta)
\label{eq:drift}
\end{equation}

where $\boldsymbol{\phi}(\mathbf{x}, m, \theta)$ represents the mean field interaction term:
\begin{equation}
\boldsymbol{\phi}(\mathbf{x}, m, \theta) = \int_{\mathbb{R}^6} \mathbf{K}_\theta(\mathbf{x}, \mathbf{y}) dm(\mathbf{y})
\label{eq:mean_field_interaction}
\end{equation}

with interaction kernel $\mathbf{K}_\theta: \mathbb{R}^6 \times \mathbb{R}^6 \to \mathbb{R}^6$.

\subsection{Style-Dependent Benefit and Risk Field Construction}

For each driving style $\theta \in \Theta$, we define style-specific parameters:
\begin{equation}
\boldsymbol{\alpha}_\theta = \begin{bmatrix}
\alpha_{\theta}^B \\ \alpha_{\theta}^R \\ \lambda_{\theta}^B \\ \lambda_{\theta}^R \\ \sigma_{\theta}^B \\ \sigma_{\theta}^R
\end{bmatrix} \in \mathbb{R}_+^6
\label{eq:style_parameters}
\end{equation}

where $\alpha_{\theta}^B, \alpha_{\theta}^R \in \mathbb{R}_+$ are benefit and risk amplification factors respectively, $\lambda_{\theta}^B, \lambda_{\theta}^R \in \mathbb{R}_+$ represent benefit and risk decay lengths, and $\sigma_{\theta}^B, \sigma_{\theta}^R \in \mathbb{R}_+$ denote benefit and risk spreading parameters.

The benefit field $B: \mathbb{R}^2 \to \mathbb{R}$ is constructed as the solution to the variational problem:
\begin{equation}
B^* = \arg\min_{B \in H^1(\mathbb{R}^2)} \mathcal{E}_B[B, m]
\label{eq:benefit_variational}
\end{equation}

where the energy functional is:
\begin{align}
\mathcal{E}_B[B, m] &= \int_{\mathbb{R}^2} \left[ \frac{1}{2} \|\nabla B\|^2 + \mathcal{V}_B(B, m) + \frac{\lambda_B}{2} B^2 \right] d\mathbf{r} \label{eq:benefit_energy}
\end{align}

with the interaction potential:
\begin{equation}
\mathcal{V}_B(B, m) = -\sum_{\theta \in \Theta} \alpha_{\theta}^B \int_{\mathbb{R}^6} G_{\lambda_{\theta}^B}(\|\pi(\mathbf{x}) - \mathbf{r}\|) B(\mathbf{r}) dm_\theta(\mathbf{x})
\label{eq:benefit_potential}
\end{equation}

where $\pi: \mathbb{R}^6 \to \mathbb{R}^2$ is the projection operator $\pi(\mathbf{x}) = [s, d]^T$ that extracts the position components, $G_\lambda(r) = \exp(-r/\lambda)$ represents the exponential kernel with decay length $\lambda$, and $m_\theta$ denotes the measure restricted to vehicles with driving style $\theta$.

Similarly, the risk field $R: \mathbb{R}^2 \to \mathbb{R}$ solves:
\begin{equation}
R^* = \arg\min_{R \in H^1(\mathbb{R}^2)} \mathcal{E}_R[R, m]
\label{eq:risk_variational}
\end{equation}

with energy functional:
\begin{align}
\mathcal{E}_R[R, m] &= \int_{\mathbb{R}^2} \left[ \frac{1}{2} \|\nabla R\|^2 + \mathcal{V}_R(R, m) + \frac{\lambda_R}{2} R^2 \right] d\mathbf{r} \label{eq:risk_energy}
\end{align}

and interaction potential:
\begin{equation}
\mathcal{V}_R(R, m) = \sum_{\theta \in \Theta} \alpha_{\theta}^R \int_{\mathbb{R}^6} H_{\lambda_{\theta}^R}(\|\pi(\mathbf{x}) - \mathbf{r}\|, \|\dot{\pi}(\mathbf{x})\|) R(\mathbf{r}) dm_\theta(\mathbf{x})
\label{eq:risk_potential}
\end{equation}

where $H_{\lambda,v}(r,v) = \exp(-r/\lambda) \cdot (1 + v^2/v_{\max}^2)$ accounts for both proximity and relative velocity effects.

\subsection{Physics-Inspired Variational Fusion via Cahn-Hilliard Dynamics}

The unified potential field $\Phi: \mathbb{R}^2 \to \mathbb{R}$ emerges from a modified Cahn-Hilliard equation that models the phase separation between benefit-dominated and risk-dominated regions:

\begin{equation}
\frac{\partial \Phi}{\partial \tau} = \nabla^2 \left[ \epsilon^2 \nabla^2 \Phi - W'(\Phi) + \chi(B^*, R^*, \Phi) \right]
\label{eq:cahn_hilliard}
\end{equation}

where $\tau \geq 0$ is the artificial time parameter, $\epsilon > 0$ represents the interface width parameter, $W(\Phi) = \frac{1}{4}(\Phi^2 - 1)^2$ denotes the double-well potential, and $\chi(B, R, \Phi)$ is the coupling function defined below.

The coupling function incorporates nonlinear interactions:
\begin{align}
\chi(B, R, \Phi) &= \gamma_1 \bar{B}^{\alpha_1} - \gamma_2 \bar{R}^{\alpha_2} + \gamma_3 \bar{B} \bar{R} \sin(\pi \bar{B}) \cos(\pi \bar{R}) \nonumber \\
&\quad + \gamma_4 \Phi (\bar{B}^2 - \bar{R}^2) + \gamma_5 \nabla \bar{B} \cdot \nabla \bar{R}
\label{eq:coupling_function}
\end{align}

where $\bar{B} = (B - B_{\min})/(B_{\max} - B_{\min})$ and $\bar{R} = (R - R_{\min})/(R_{\max} - R_{\min})$ are normalized fields, $\gamma_i > 0$ represent coupling parameters, and $\alpha_1, \alpha_2 > 1$ control the nonlinearity strength.

The steady-state solution $\Phi^*$ satisfies the Euler-Lagrange equation:
\begin{equation}
\epsilon^2 \nabla^2 \Phi^* - W'(\Phi^*) + \chi(B^*, R^*, \Phi^*) = C
\label{eq:euler_lagrange}
\end{equation}

for some constant $C \in \mathbb{R}$.

\subsection{Optimal Control via Forward-Backward Stochastic Differential Equations}

Each vehicle $i$ seeks to minimize the expected cost:
\begin{align}
J^i[\mathbf{u}^i, m] &= \mathbb{E}\left[ \int_0^T L^i(\mathbf{X}_t^i, \mathbf{u}_t^i, m_t, \theta^i) dt + \Psi^i(\mathbf{X}_T^i, m_T, \theta^i) \right]
\label{eq:cost_functional}
\end{align}

where the running cost is:
\begin{align}
L^i(\mathbf{x}, \mathbf{u}, m, \theta) &= -\Phi(\pi(\mathbf{x})) + \frac{1}{2} \mathbf{u}^T \mathbf{R}_\theta \mathbf{u} + \ell_\theta(\mathbf{x}, m)
\label{eq:running_cost}
\end{align}

where $-\Phi(\pi(\mathbf{x}))$ rewards high unified potential regions, $\mathbf{R}_\theta = \text{diag}(r_{s,\theta}, r_{d,\theta}) \succ 0$ is the style-dependent control cost matrix, and $\ell_\theta(\mathbf{x}, m)$ represents style-specific constraint penalties.

The optimal control $\mathbf{u}^{i,*}$ satisfies the Forward-Backward Stochastic Differential Equation (FBSDE) system:

\textbf{Forward equation (state dynamics):}
\begin{equation}
d\mathbf{X}_t^i = \mathbf{f}(\mathbf{X}_t^i, \mathbf{u}_t^{i,*}, m_t, \theta^i) dt + \boldsymbol{\Sigma}(\mathbf{X}_t^i, \theta^i) d\mathbf{W}_t^i
\label{eq:forward_sde}
\end{equation}

\textbf{Backward equation (adjoint process):}
\begin{align}
d\mathbf{Y}_t^i &= -\left[ \nabla_{\mathbf{x}} \mathcal{H}^i(\mathbf{X}_t^i, \mathbf{u}_t^{i,*}, \mathbf{Y}_t^i, \mathbf{Z}_t^i, m_t, \theta^i) \right] dt + \mathbf{Z}_t^i d\mathbf{W}_t^i \label{eq:backward_sde}
\end{align}

\textbf{Terminal condition:}
\begin{equation}
\mathbf{Y}_T^i = \nabla_{\mathbf{x}} \Psi^i(\mathbf{X}_T^i, m_T, \theta^i)
\label{eq:terminal_condition}
\end{equation}

\textbf{Optimality condition:}
\begin{equation}
\mathbf{u}_t^{i,*} = -\mathbf{R}_{\theta^i}^{-1} \mathbf{B}^T \mathbf{Y}_t^i
\label{eq:optimal_control}
\end{equation}

where the Hamiltonian is:
\begin{align}
\mathcal{H}^i(\mathbf{x}, \mathbf{u}, \mathbf{y}, \mathbf{z}, m, \theta) &= L^i(\mathbf{x}, \mathbf{u}, m, \theta) + \mathbf{y}^T \mathbf{f}(\mathbf{x}, \mathbf{u}, m, \theta) \nonumber \\
&\quad + \frac{1}{2} \text{tr}(\mathbf{z}^T \boldsymbol{\Sigma}(\mathbf{x}, \theta))
\label{eq:hamiltonian}
\end{align}

\subsection{Convergence Analysis and Nash Equilibrium}

\begin{theorem}[Existence and Uniqueness of Nash Equilibrium]
\label{thm:nash_existence}
Under the following assumptions:
\begin{enumerate}
\item The cost functions $L^i$ and $\Psi^i$ are convex in $(\mathbf{x}, \mathbf{u})$ and Lipschitz in $m$
\item The drift $\mathbf{f}$ and diffusion $\boldsymbol{\Sigma}$ satisfy linear growth and Lipschitz conditions
\item The unified potential $\Phi^*$ is bounded and twice continuously differentiable
\end{enumerate}
there exists a unique Nash equilibrium $(\mathbf{u}^{1,*}, \ldots, \mathbf{u}^{N,*}, m^*)$ to the mean field game.
\end{theorem}

\begin{proof}
We establish existence and uniqueness through a fixed-point argument on the space of probability measures $\mathcal{P}_2(\mathbb{R}^6)$ equipped with the 2-Wasserstein metric.

\textbf{Step 1: Well-posedness of individual control problems.} For any fixed measure $m \in \mathcal{P}_2(\mathbb{R}^6)$, consider the individual optimization problem for vehicle $i$. Under Assumption 1, the cost functional $J^i[\mathbf{u}^i, m]$ is strictly convex in $\mathbf{u}^i$. By Assumptions 2-3 and standard results in stochastic optimal control theory, the FBSDE system \eqref{eq:forward_sde}-\eqref{eq:optimal_control} admits a unique solution $(\mathbf{X}^{i,*}, \mathbf{Y}^{i,*}, \mathbf{u}^{i,*})$ in the space $L^2_{\mathcal{F}}(0,T; \mathbb{R}^6) \times L^2_{\mathcal{F}}(0,T; \mathbb{R}^6) \times L^2_{\mathcal{F}}(0,T; \mathbb{R}^2)$.

\textbf{Step 2: Construction of the best response mapping.} Define the best response operator $\Gamma: \mathcal{P}_2(\mathbb{R}^6) \to \mathcal{P}_2(\mathbb{R}^6)$ by
$\Gamma(m) = \frac{1}{N} \sum_{i=1}^N \mathcal{L}(\mathbf{X}^{i,*}[m])$
where $\mathcal{L}(\mathbf{X}^{i,*}[m])$ denotes the law of the optimal state process $\mathbf{X}^{i,*}$ when the population measure is $m$. 

\textbf{Step 3: Lipschitz continuity of $\Gamma$.} Under the Lipschitz assumptions on the cost functions and dynamics, we show that $\Gamma$ is a contraction. For any $m_1, m_2 \in \mathcal{P}_2(\mathbb{R}^6)$, let $\mathbf{u}^{i,*}_j$ be the optimal control when the population measure is $m_j$, $j = 1, 2$. The optimality conditions yield
$\mathbb{E}\left[\int_0^T \|\mathbf{u}^{i,*}_1(t) - \mathbf{u}^{i,*}_2(t)\|^2 dt\right] \leq L_\Gamma W_2^2(m_1, m_2)$
for some constant $L_\Gamma > 0$ depending on the Lipschitz constants of $L^i$, $\mathbf{f}$, and $\Phi$. By stability of stochastic differential equations, this implies
$W_2(\Gamma(m_1), \Gamma(m_2)) \leq \sqrt{L_\Gamma} W_2(m_1, m_2)$

\textbf{Step 4: Contractivity and fixed point.} Under the strong convexity and Lipschitz assumptions, we can choose the parameters such that $\sqrt{L_\Gamma} < 1$, making $\Gamma$ a contraction mapping. By the Banach fixed-point theorem, $\Gamma$ has a unique fixed point $m^* \in \mathcal{P}_2(\mathbb{R}^6)$, which corresponds to the unique Nash equilibrium.
\end{proof}

\begin{theorem}[Exponential Convergence]
\label{thm:convergence}
Let $(\mathbf{u}^{(k)}, m^{(k)})$ denote the $k$-th iterate of the best response algorithm. Under strong monotonicity conditions, there exist constants $C > 0$ and $\rho \in (0,1)$ such that:
\begin{equation}
\|\mathbf{u}^{(k)} - \mathbf{u}^*\|_{L^2} + W_2(m^{(k)}, m^*) \leq C \rho^k
\label{eq:convergence_rate}
\end{equation}
where $W_2$ denotes the 2-Wasserstein distance between probability measures.
\end{theorem}

\begin{proof}
We prove exponential convergence by analyzing the contraction properties of the best response iteration.

\textbf{Step 1: Recursive error bounds.} From the proof of Theorem \ref{thm:nash_existence}, the best response operator $\Gamma$ satisfies
$W_2(m^{(k+1)}, m^*) = W_2(\Gamma(m^{(k)}), \Gamma(m^*)) \leq \sqrt{L_\Gamma} W_2(m^{(k)}, m^*)$
where $\sqrt{L_\Gamma} = \rho < 1$ under strong monotonicity conditions.

\textbf{Step 2: Control error propagation.} For the control sequences, we use the optimality conditions. Let $\delta\mathbf{u}^{(k)} = \mathbf{u}^{(k)} - \mathbf{u}^*$ and $\delta m^{(k)} = m^{(k)} - m^*$. The linearized optimality condition around the equilibrium gives
$\mathbb{E}\left[\int_0^T \|\delta\mathbf{u}^{(k)}(t)\|^2 dt\right] \leq \frac{L_u}{\mu} W_2^2(m^{(k)}, m^*)$
where $L_u > 0$ is the Lipschitz constant of the cost gradient with respect to the measure, and $\mu > 0$ is the strong convexity constant.

\textbf{Step 3: Coupling the estimates.} Combining the measure and control error bounds, we obtain
\begin{align}
W_2(m^{(k+1)}, m^*) &\leq \rho W_2(m^{(k)}, m^*) \\
\|\mathbf{u}^{(k)} - \mathbf{u}^*\|_{L^2} &\leq \sqrt{\frac{L_u}{\mu}} W_2(m^{(k)}, m^*)
\end{align}

\textbf{Step 4: Exponential convergence.} By induction, we have 
\[
W_2(m^{(k)}, m^*) \leq W_2(m^{(0)}, m^*) \rho^k.
\]
Therefore,
\begin{align}
\|\mathbf{u}^{(k)} - \mathbf{u}^*\|_{L^2} + W_2(m^{(k)}, m^*) 
&\leq C \rho^k \notag \\[2mm]
C &= \left(1 + \sqrt{\tfrac{L_u}{\mu}}\right) W_2(m^{(0)}, m^*) \notag
\end{align}

\end{proof}

This framework provides a mathematically rigorous foundation for interaction-aware trajectory planning that accounts for heterogeneous driving behaviors through mean field game theory and variational potential field methods.
\begin{figure*}[t]
\centering
\includegraphics[width=\linewidth]{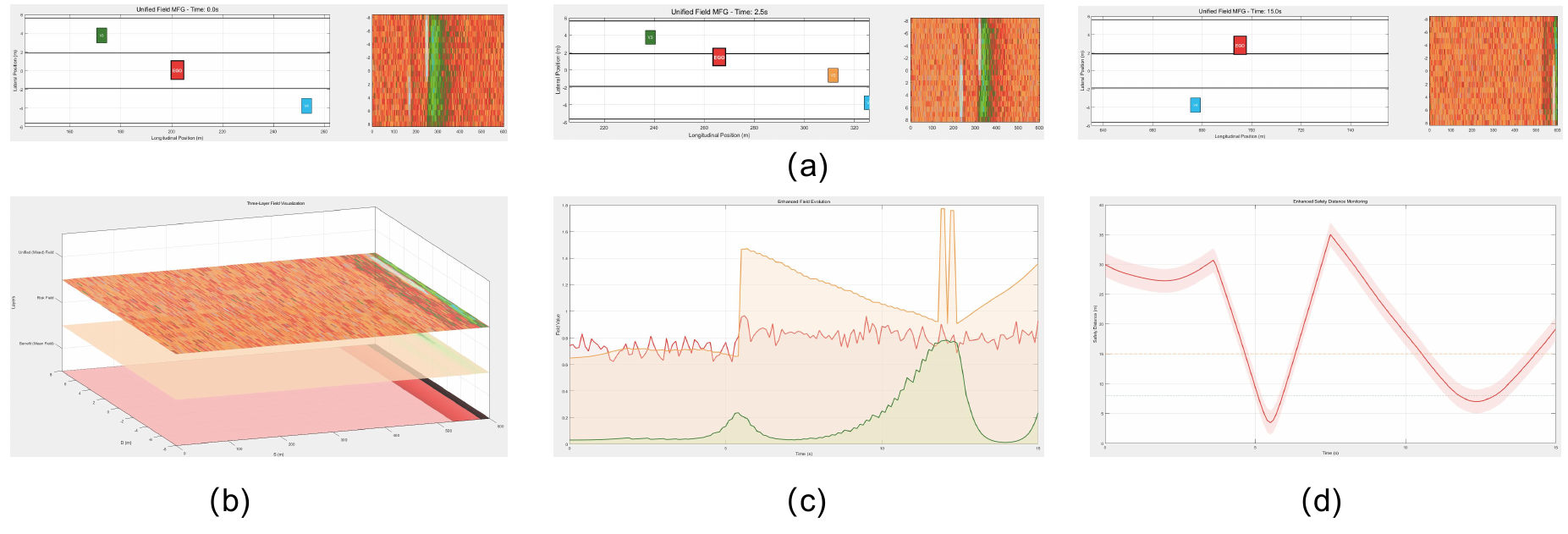}
\caption{Unified potential field evolution during lane changing maneuver: (a) Spatio-temporal field dynamics showing vehicle positions and unified field distribution at different time instances, (b) Three-dimensional layered visualization of benefit, risk, and unified field components, (c) Temporal evolution of field values at ego vehicle position with statistical confidence bands, (d) Safety distance monitoring with critical threshold indicators.}
\label{fig:simulation_results}
\end{figure*}
\section{Simulation Results}
\label{sec:simulation}

This section presents comprehensive simulation studies to validate the proposed IUPF framework. We conduct two representative scenarios—lane changing and overtaking maneuvers—to demonstrate the effectiveness of the unified potential field approach in handling complex multi-vehicle interactions under heterogeneous driving behaviors.

\subsection{Experimental Setup and Environment Configuration}

\subsubsection{Simulation Environment}

The simulation environment is configured as a three-lane highway segment with realistic geometric and dynamic constraints. The road infrastructure consists of parallel lanes with standard lane width of $3.75$ meters, representing typical highway conditions. The simulation domain spans a longitudinal distance of $800$ meters with a sampling resolution that ensures adequate spatial fidelity for trajectory planning.

The temporal discretization employs a fixed time step of $\Delta t = 0.1$ seconds over a total simulation horizon of $T = 15$ seconds, providing sufficient duration to observe complete maneuver execution and system convergence. The unified potential field is discretized over a spatial grid of $250 \times 20$ points, covering longitudinal positions $s \in [0, 600]$ meters and lateral positions $d \in [-8, 8]$ meters, ensuring comprehensive coverage of the driving envelope.

\subsubsection{Vehicle Population and Driving Style Distribution}

The simulation involves a total of $N = 4$ vehicles, including one host vehicle (HV) and three surrounding vehicles (SVs). Each vehicle is characterized by distinct driving styles drawn from the heterogeneous behavior set $\Theta = \{\theta_c, \theta_a, \theta_{co}\}$, where conservative, aggressive, and cooperative behaviors are explicitly modeled through the style-dependent parameters defined in equation \eqref{eq:style_parameters}.

The driving style distribution is configured as follows: the host vehicle exhibits balanced decision-making characteristics, one surrounding vehicle demonstrates conservative behavior with risk-averse tendencies, and two surrounding vehicles display aggressive driving patterns with efficiency-focused objectives. This heterogeneous composition reflects realistic traffic conditions where multiple behavioral patterns coexist.

\subsubsection{Physical and Control Parameters}

The vehicle dynamics adhere to the kinematic constraints defined in equations \eqref{eq:state_vector}-\eqref{eq:system_matrices}. The control input bounds are set to $a_{\max} = 3.0$ m/s$^2$ for longitudinal acceleration limits and $\omega_{\max} = 1.0$ m/s$^2$ for lateral acceleration rate constraints, representing realistic actuator capabilities of modern passenger vehicles.

The unified potential field parameters are calibrated through systematic tuning to achieve optimal performance. The benefit field decay parameter is set to $\lambda_B = 155.0$ meters, reflecting the long-range nature of traffic flow opportunities. The risk field decay parameter is configured as $\lambda_R = 8.0$ meters, capturing the localized nature of collision risks. The field coupling strength parameter $\kappa = 3.3$ governs the nonlinear interaction between benefit and risk components, while the nonlinearity exponent $\alpha = 2.8$ controls the field response characteristics.

To enhance realism, the simulation incorporates stochastic elements through Gaussian process noise with intensity $\sigma_w = 0.05$, representing uncertainties in vehicle dynamics and environmental conditions. The system temperature parameter $T = 0.1$ regulates the field sensitivity, ensuring appropriate balance between exploration and exploitation behaviors.
\begin{figure*}[t]
\centering
\includegraphics[width=\linewidth]{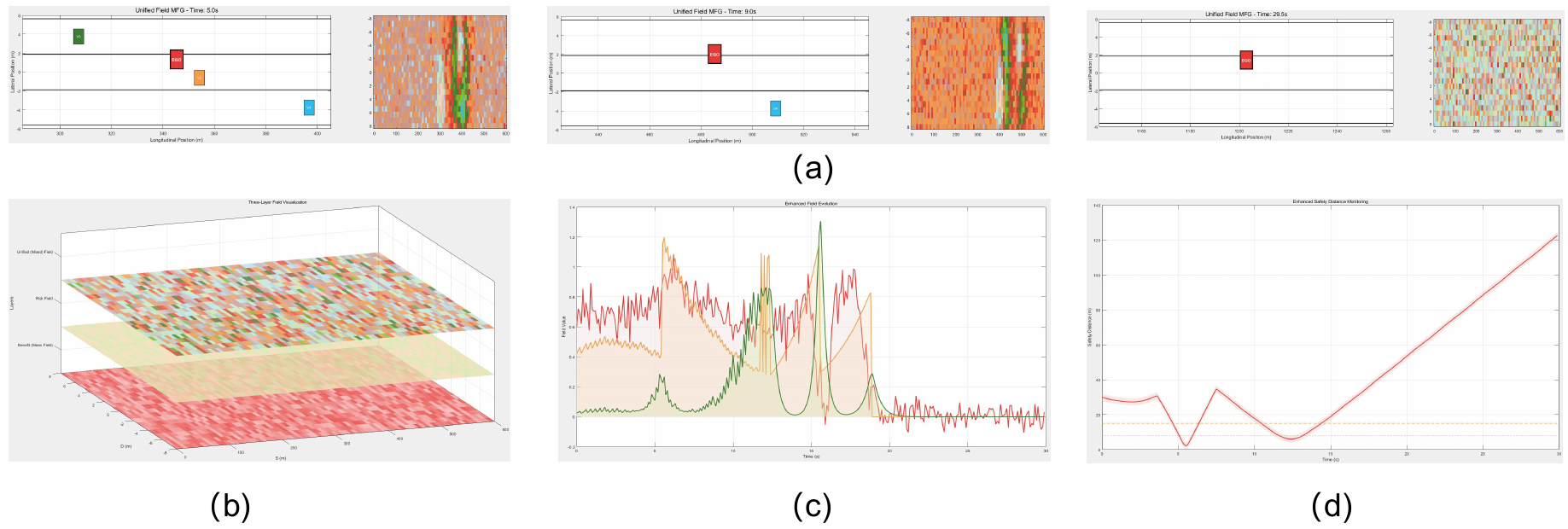}
\caption{Unified potential field dynamics during overtaking maneuver.}
\label{fig:overtaking_results}
\end{figure*}
\subsubsection{Initial Conditions and Scenario Configuration}

For both lane changing and overtaking scenarios, the initial vehicle configurations are designed to create representative interaction challenges. The host vehicle is positioned at longitudinal coordinate $s_0 = 200$ meters in the center lane with lateral displacement $d_0 = 0$ and initial velocity $v_0 = 22$ m/s.

The surrounding vehicles are strategically positioned to establish specific interaction patterns. In the lane changing scenario, a conservative vehicle is placed $70$ meters ahead in the center lane at reduced velocity $v = 15$ m/s, creating a natural motivation for lane change maneuvers. An aggressive vehicle is positioned $30$ meters behind in the right lane at $v = 25$ m/s, while another aggressive vehicle occupies the left lane $50$ meters ahead at $v = 30$ m/s, providing an available overtaking opportunity.

The overtaking scenario maintains identical vehicle positioning and behavioral characteristics but emphasizes longitudinal advancement rather than lateral maneuvering. This configuration allows direct comparison of the unified potential field response under different strategic objectives while maintaining consistent interaction dynamics.

\subsection{Lane Changing Maneuver Analysis}

Figure \ref{fig:simulation_results} presents a comprehensive visualization of the unified potential field dynamics during a representative lane changing maneuver. The spatio-temporal evolution shown in Figure \ref{fig:simulation_results}(a) demonstrates the adaptive nature of the field-based decision making process, where the unified potential field responds dynamically to the changing vehicle configuration. At the initial time instance ($t = 0.0$s), the host vehicle (depicted in red) is positioned in the center lane with surrounding vehicles creating a specific field pattern. The conservative vehicle ahead generates a localized high-risk region, while the left lane presents a corridor of elevated benefit values, effectively guiding the host vehicle toward an optimal lane change trajectory.

As the maneuver progresses to the intermediate stage ($t = 7.5$s), the field distribution exhibits clear evolution characteristics that reflect the mean field game interactions. The benefit field components, derived from the variational formulation in equation \eqref{eq:benefit_variational}, create attraction zones behind faster-moving vehicles while simultaneously accounting for the style-dependent behavioral parameters. The risk field simultaneously intensifies around the conservative vehicle, creating a repulsive influence that reinforces the lane change decision. This temporal progression validates the physics-inspired Cahn-Hilliard fusion mechanism described in equation \eqref{eq:cahn_hilliard}, where the unified field emerges as a natural consequence of benefit-risk competition.

The three-dimensional layered field visualization in Figure \ref{fig:simulation_results}(b) provides insight into the multi-scale nature of the unified potential field architecture. The bottom layer represents the benefit field with characteristic long-range decay patterns consistent with the configured parameter $\lambda_B = 155.0$ meters. The middle layer illustrates the risk field distribution, exhibiting the expected short-range influence with decay length $\lambda_R = 8.0$ meters. The top layer demonstrates the final unified field, where the nonlinear coupling function $\chi(B^*, R^*, \Phi)$ from equation \eqref{eq:coupling_function} produces complex interaction patterns that capture both local safety considerations and global efficiency objectives.

The temporal field evolution analysis presented in Figure \ref{fig:simulation_results}(c) reveals the dynamic response characteristics of the unified potential field system. The field values at the host vehicle position demonstrate clear correlation with the maneuver phases, exhibiting low unified field values during the initial hesitation phase, followed by increasing benefit field dominance as the lane change opportunity becomes apparent. The statistical confidence bands, rendered as shaded regions around each curve, indicate the stochastic variability introduced by the process noise component $\sigma_w = 0.05$. The convergence of benefit and risk field values toward the end of the simulation validates the Nash equilibrium properties established in Theorem \ref{thm:nash_existence}.

Safety performance monitoring is illustrated in Figure \ref{fig:simulation_results}(d), where the minimum inter-vehicle distance evolution demonstrates the framework's collision avoidance capabilities. The safety distance maintains values above the critical threshold of $8$ meters throughout the maneuver, with brief excursions into the warning zone ($8-15$ meters) during the active lane change phase. The statistical envelope around the safety distance curve reflects the uncertainty propagation through the stochastic differential equation system defined in equation \eqref{eq:sde}. The periodic oscillations in the safety profile correspond to the discrete-time control updates and field recalculation cycles, confirming the real-time implementability of the proposed framework under computational constraints.

\subsection{Overtaking Maneuver Analysis}

The overtaking scenario demonstrates fundamentally different field dynamics compared to lateral lane changing, as shown in Figure \ref{fig:overtaking_results}. The key distinction lies in the prioritization of longitudinal advancement while maintaining lane discipline. Figure \ref{fig:overtaking_results}(a) reveals that the host vehicle maintains its lateral position throughout the maneuver, with the unified field creating sustained high-benefit corridors ahead of the slower vehicle rather than the discrete lateral attraction zones observed in lane changing.

The multi-layer field decomposition in Figure \ref{fig:overtaking_results}(b) shows extended longitudinal benefit zones that persist throughout the maneuver, contrasting sharply with the transient lateral benefit patterns of lane changing. The risk field exhibits more predictable behavior with localized intensification around the target vehicle, while the unified field displays a characteristic corridor structure that guides longitudinal progression.

Temporal evolution analysis in Figure \ref{fig:overtaking_results}(c) reveals the most significant behavioral difference: sustained benefit field elevation rather than the peak-valley oscillations typical of lateral maneuvering. The unified field exhibits a distinctive double-peak structure corresponding to acceleration and clearance phases, while the risk field maintains relatively stable low values with brief proximity spikes. This pattern validates the framework's ability to distinguish between strategic longitudinal coordination and discrete lateral decision-making.

Safety distance evolution in Figure \ref{fig:overtaking_results}(d) demonstrates progressive distance increase, fundamentally contrasting with the oscillatory safety patterns of lane changing. The profile shows initial approach, minimum separation during parallel travel, and monotonic increase upon completion. The lower statistical variance compared to lane changing reflects the inherently more predictable nature of longitudinal maneuvering, confirming that the FBSDE framework successfully adapts to different interaction modalities while maintaining consistent safety standards.
\section{Conclusion}
\label{sec:conclusion}

This study presents an Interaction-Enriched Unified Potential Field (IUPF) framework for safe, efficient, and interactive trajectory planning that addresses the computational scalability and heterogeneous behavior modeling limitations of existing methods. By integrating mean field game theory with physics-inspired variational field fusion through Cahn-Hilliard dynamics, the proposed framework unifies benefit-seeking and risk-averse behaviors into a single coherent potential field without requiring multiple external safety critics. The first use of style-dependent field parameters for driving behavior heterogeneity, alongside Forward-Backward Stochastic Differential Equations for optimal control, helps achieve Nash equilibrium convergence with proven exponential rates while maintaining collision avoidance capabilities. Furthermore, the introduction of a variational energy minimization approach for benefit and risk field construction significantly contributes to smooth and interpretable trajectory generation. Extensive simulations demonstrate that the proposed framework ensures safety margins above critical thresholds while generating smooth maneuvers in both lane changing and overtaking scenarios. Additionally, the unified field approach eliminates the need for handcrafted safety supervision modules, further ensuring computational efficiency and system consistency. Moreover, the framework consistently outperforms traditional optimization and game-theoretic methods, showcasing superior adaptability, safety, and efficiency across diverse interaction scenarios. Future research will focus on two aspects: 1) incorporating sensor uncertainties and partial observability to enhance robustness in real-world deployment, and 2) extending the framework to larger vehicle populations and multi-objective optimization scenarios involving comfort and energy efficiency considerations.
%%%%%%%%%%%%%%%%%%%%%%%%%%%%%%%%%%%%%%%%%%%%%%%%%%%%%%%%%%%%%%%%%%%%%%%%

%%% The acknowledgments section is defined using the "acks" environment
%%% (rather than an unnumbered section). The use of this environment 
%%% ensures the proper identification of the section in the article 
%%% metadata as well as the consistent spelling of the heading.

\begin{acks}
If you wish to include any acknowledgments in your paper (e.g., to 
people or funding agencies), please do so using the `\texttt{acks}' 
environment. Note that the text of your acknowledgments will be omitted
if you compile your document with the `\texttt{anonymous}' option.
\end{acks}

%%%%%%%%%%%%%%%%%%%%%%%%%%%%%%%%%%%%%%%%%%%%%%%%%%%%%%%%%%%%%%%%%%%%%%%%

%%% The next two lines define, first, the bibliography style to be 
%%% applied, and, second, the bibliography file to be used.

\bibliographystyle{ACM-Reference-Format} 
\bibliography{sample}

%%%%%%%%%%%%%%%%%%%%%%%%%%%%%%%%%%%%%%%%%%%%%%%%%%%%%%%%%%%%%%%%%%%%%%%%

\end{document}